\documentclass[submission,copyright,creativecommons]{eptcs}
\providecommand{\event}{SOS 2007} % Name of the event you are submitting to
\usepackage{amsmath}
\usepackage{amssymb}
\usepackage{amsfonts}
\usepackage{amsthm}
\usepackage{iftex}
\usepackage{xcolor}
\usepackage{tikz}
\usepackage{comment}

\usepackage{color}

\ifpdf
  \usepackage{underscore}         % Only needed if you use pdflatex.
  \usepackage[T1]{fontenc}        % Recommended with pdflatex
\else
  \usepackage{breakurl}           % Not needed if you use pdflatex only.
\fi

\title{Causal Kripke Models\footnote{This project has received funding from the European Union’s Horizon 2020 research and innovation programme under the Marie Skłodowska-Curie grant agreement No 101007627.}}

\author{Yiwen Ding$^2$ \footnote{Yiwen Ding, Ruoding Wang and Xiaolong Wang are supported by the China Scholarship Council}
\and
 Krishna Manoorkar$^2$ \thanks{Krishna Manoorkar is supported by the NWO grant KIVI.2019.001 awarded to Alessandra Palmigiano.}
\and
 Apostolos Tzimoulis$^2$
\and 
 Ruoding Wang$^{2,3\,\small{\dagger}}$
\and
 Xiaolong Wang$^{1,2\,\small{\dagger}}$
 \institute{
Shandong University$^1$ $\quad$  Vrije  Universiteit, Amsterdam$^2$ $\quad$  Xiamen University$^3$}
}
\newtheorem{theorem}{Theorem}[section]

\newtheorem{definition}[theorem]{Definition}
\newtheorem{example}[theorem]{Example}

\newcommand{\marginnote}[1]{\marginpar{\raggedright\tiny{#1}}}

\def\titlerunning{Causal Kripke Models}
\def\authorrunning{Ding, Manoorkar, Tzimoulis, Wang, Wang}
\begin{document}
\maketitle

\begin{abstract}
This work  extends Halpern and Pearl's causal models for actual causality to a possible world semantics environment. Using this framework we introduce a logic of actual causality with modal operators, which allows for reasoning about causality in scenarios involving multiple  possibilities, temporality,  knowledge and uncertainty. %We also provide a sound and complete axiomatization of this logic.
We illustrate this with a number of examples, and conclude by discussing some future directions for research. 
\end{abstract}

\section{Introduction}\label{sec:introduction}

%Causality plays a central role in human reasoning and knowledge. Defining and formalizing causality has been a key area of research in philosophy and formal methods \cite{falcon2006aristotle,hume2007enquiry,lewis1974causation, dubois2020glance}. In recent years with the advent of machine learning and AI there has been increasing interest in formalizing causal reasoning. Designing algorithms which can understand causal information and perform causal reasoning is one of the key areas of AI research \cite{bergstein2020ai,pearl2018book,scholkopf2022causality}. Causal reasoning can play important roles in formally modelling notions like responsibility,  blame, harm and explanation important aspects in designing ethical and responsible AI systems \cite{beckers2022causal}. 

Causality is crucial in human reasoning and knowledge. Defining and formalizing causality has been a significant area of research in philosophy and formal methods \cite{falcon2006aristotle,hume2007enquiry,lewis1974causation, dubois2020glance}. In recent years, with the rise of machine learning and AI, there has been growing interest in formalizing causal reasoning. One of the key areas of AI research is designing algorithms capable of comprehending causal information and performing causal reasoning \cite{bergstein2020ai,pearl2018book,scholkopf2022causality}. Causal reasoning can be instrumental in formally modeling notions such as responsibility, blame, harm, and explanation, which are important aspects in designing ethical and responsible AI systems \cite{beckers2022causal}.

In this article we focus on the kind of causality known as "actual causality" (a.k.a. token causality) \cite{dawid2007fundamentals,hitchcock2013cause,hitchcock2001tale,woodward2005making}. Actual causality refers to the causality of a specific event which has actually happened (e.g.~"John died because Alice shot him") rather than general causes (e.g.~"smoking causes cancer"). Several formal approaches have been used for modelling actual causality \cite{lewis1974causation,lewis2004causation,glymour2007actual,hall2007structural}. One of the most prominent formalizations of actual causation was developed by Halpern and Pearl \cite{pearl2009causality,halpern2005causes,halpern2005causesII}. This model describes dependencies between {\em endogenous variables} and {\em exogenous variables} using {\em structural equations}. Based on causal models Halpern and Pearl have given three different definitions of actual causality known as {\em original, updated and modified} definitions \cite{halpern2005causes,halpern2005causesII,halpern2015modification} of actual causality using counterfactual reasoning. The formal language developed to describe actual causality in this model is used to define several notions like  normality,  blame, accountability and responsibility. This model has been used in several applications in law \cite{moore2009causation},   database  theory \cite{meliou2010causality},   model checking \cite{chockler2008causes,datta2015program,beer2012explaining}, and AI \cite{ibrahim2020actual,dubois2020glance,beckers2022causal}.

Notions like  knowledge, temporality, possibility, normality (or typicality) and uncertainty play important role in causal reasoning and related applications. In the past, attempts have been made to incorporate some of these notions into the causal models of Halpern and Pearl. In~\cite{beer2012explaining},  Beer et.al.~define causality in linear temporal logic to explain counterexamples. This line of research has been carried forward in model checking and program verification~\cite{baier2021verification,10.1007/978-3-642-35873-9}.
The Halpern and Pearl formalism has also been extended to define causality in frameworks such as transition systems and Hennessy-Milner logic~\cite{caltais2016composing,caltais2020causal,baier2021verification}. In \cite{barberoobserving},  Barbero et.al.~define  causality with epistemic operators. However, to the best of our knowledge a general Kripke model for actual causality based on Halpern and Pearl framework has not been studied yet.

In this work, we develop the notion of causal Kripke models and introduce a modal language for causal reasoning with uncertainty, temporality, possibility, and epistemic knowledge. We show that our model can formalize notions like sufficient causality, blame, responsibility, normality, and explanations. Our framework provides a more natural definition of sufficient causality \cite[Section 2.6]{halpern2016actual}  by considering nearby contexts, which Halpern's causal model does not support
%Halpern discusses the need to consider nearby contexts in defining sufficient causality
(c.f. \ref{ssec:sufficint causality}). The developed causal Kripke models offer a straightforward way to describe nearby contexts and define sufficient causality as intended by Halpern. In order to stay as close as possible to Halpern's original framework, where formally only atomic events can be causes, we utilize a hybrid language not only contains modalities but also names for the possible worlds.

The structure of the paper is as follows. In Section \ref{sec:preliminaries}, we provide preliminaries on causal models and logic of  causality. In Section \ref{sec: Motivation}, we give several examples to motivate the development of causal Kripke semantics. In Section \ref{sec: Possible}, we define causal Kripke models, and develop a modal logic of actual causality to reason about them. We generalize the Halpern-Pearl definitions of actual causality to this framework and provide a sound and complete axiomatization of the modal logic of actual causality. In Section \ref{sec:Examples}, we model the examples discussed in Section \ref{sec: Motivation} using our framework and also show how this model can be used to provide an intuitive definition of sufficient causality. Finally, in Section \ref{sec:conclusion} we conclude and provide some directions for future research.

\section{Preliminaries} \label{sec:preliminaries}
\subsection{Causal models} \label{ssec:causal models}
In this section we briefly recall key concepts and ideas of the  standard logic of causal reasoning as presented in \cite{halpern2016actual}.  A causal model describes the world in terms of variables which take values over certain sets. The variables and their ranges are given by a {\em signature} $\mathcal{S}=(\mathcal{U}, \mathcal{V}, \mathcal{R})$ where $\mathcal{U}$ is a finite set of {\em exogenous} variables (i.e., variables whose value is independent of other variables in the model), $\mathcal{V}$, which is disjoint with $\mathcal{U}$, is a finite set of {\em endogenous} variables (i.e., variables whose value is determined by  other variables in the model), and $\mathcal{R}(X)$ for any $X \in \mathcal{U} \cup \mathcal{V}$, is the (finite)  range of $X$.  These variables may have dependencies between them described by {\em structural equations} defined as follows.
\begin{definition}
  A causal model is a pair $(\mathcal{S}, \mathcal{F})$, where $\mathcal{S}=(\mathcal{U}, \mathcal{V}, \mathcal{R}) $ is the model's signature and $ \mathcal{F} = (f_{V_i} \mid V_i \in \mathcal{V})$ assigns to each endogenous  variable $V_i$ a map such that
  \begin{center}
     $f_{V_i}: \mathcal{R} (\mathcal{U}\cup  \mathcal{V} -\{V_i\}) \to \mathcal{R}(V_i). $ 
  \end{center}
\end{definition}
\begin{definition}
For any variables $V \in \mathcal{V}$ and $X \in \mathcal{U} \cup \mathcal{V}$, we say "$X$ is a direct cause, or a parent, of $V$" if there exist $x, x' \in \mathcal{R}(X)$ and $\overline{z} \in \mathcal{R}( \mathcal{U} \cup \mathcal{V} -\{X,V\}) $ such that 
$f_V(\overline{z},x) \neq f_V(\overline{z},x')$. A causal model is said to be \textbf{recursive} if it contains no cyclic dependencies. 
\end{definition}
\begin{definition}
 For a causal model $M= (\mathcal{S}, \mathcal{F})$, a \textbf{ context} $\overline{t} $  assigns every variable $U \in \mathcal{U}$ a value in $\mathcal{R}(U)$.  A \textbf{ causal setting} is a pair $(M, \overline{t})$, where $M$ is a causal model and $\overline{t} $ is a context for it. 
\end{definition}
In recursive models, as there are no cyclic dependencies the values of all endogenous variables are determined by the context. Throughout this paper we only consider recursive causal models.

\begin{definition}
    Let $M = (\mathcal{S}, \mathcal{F}) $ be some causal model and $\mathcal{Y} \subseteq \mathcal{V}$ be a set of endogenous variables. 
    Let $\overline{Y}$ be the  injective listing of the variables of $\mathcal{Y}$. Let $\overline{Y} = \overline{y}$ be an assignment such that $ y_i \in \mathcal{R}(Y_i) $ for every $Y_i \in \mathcal{Y}$. The causal model obtained from \textbf{ intervention} setting values of variables of $\overline{Y}$ to $\overline{y}$ is given by $M_{\overline{Y} \leftarrow \overline{y}} = (\mathcal{S}, \mathcal{F}_{\overline{Y} \leftarrow \overline{y}}) $ , where $\mathcal{F}_{\overline{Y} \leftarrow \overline{y}}$ is obtained by replacing for every variable $Y_i \in \mathcal{Y} $, the structural equation $f_{Y_i}$ with $Y_i=y_i$.
\end{definition}
Here, we consider the exogenous variables as given. Thus, we do not allow interventions on them. 
\subsection{Basic language for describing causality}\label{ssec:basic language}
The basic language, $\mathrm{L}_C$, for describing causality is an extension of  propositional logic where {\em primitive events}  are of the form $X=x$,  where $X \in \mathcal{V}$ is an endogenous variable and $x \in \mathcal{R}(X)$. Given the signature $\mathcal{S}=(\mathcal{U}, \mathcal{V}, \mathcal{R})$, the formulas $\phi\in\mathrm{L}_C$ are defined by the following recursion:
\begin{center}
    $\alpha ::= X=x \mid \neg \alpha \mid \alpha \wedge \alpha  \quad \text{where}, \, X \in \mathcal{V}, \,  x \in \mathcal{R}(X)$
$
\phi::=  X=x  \mid \neg \phi \mid \phi \wedge \phi \mid [\overline{Y}\leftarrow\overline{y}]\alpha \quad \text{where}, \, \overline{Y}\leftarrow\overline{y} \, \text{is an intervention}$
\end{center}

For any causal setting $(M,\overline{t})$ and formula $\phi \in \mathrm{L}_C$, the satisfaction relation $(M,\overline{t}) \Vdash \phi$ is defined as follows. For any formula $X=x$, $(M,\overline{t}) \Vdash X=x$ if the value of endogenous variable $X$ is set to $x$ in context $\overline{t}$. Satisfaction for the Boolean connectives is defined in a standard manner. Satisfaction of intervention formulas is defined as follows: for any  event $\alpha$,  $(M,\overline{t}) \Vdash [\overline{Y}\leftarrow\overline{y}]\alpha $ iff  $(M_{\overline{Y}\leftarrow\overline{y}},\overline{t}) \Vdash \alpha$. This language is used by Halpern and Pearl to provide three different definitions of  causality  referred as {\em original} , {\em updated} and {\em modified}  definitions of causality \cite[Section 2.2]{halpern2016actual} (for details see Appendix \ref{def: causality original}).

\iffalse
We extend the language $\mathrm{L}_C$ using causal formulas.

For any set of variables $\overline{X} \subseteq \mathcal{V}$,  we use $cause (\overline{X}=\overline{x}, \alpha) $ (resp.~$cause^u (\overline{X}=\overline{x}, \alpha)$, $cause^m (\overline{X}=\overline{x}, \alpha)$), is a causal formula describing $\overline{X}=\overline{x}$ is a cause of $\alpha$ by the original (resp.~updated, modified) definition. Satisfaction clauses for these formulas are defined as follows: $ (M, \overline{t}) \Vdash cause (\overline{X}=\overline{x}, \alpha)$ (resp.~$ (M, \overline{t}) \Vdash cause^u(\overline{X}=\overline{x}, \alpha)$, $ (M, \overline{t}) \Vdash cause^m (\overline{X}=\overline{x}, \alpha)$ )  to denote $\overline{X}=\overline{x}$ is a cause of $\alpha$ in causal setting $ (M, \overline{t}) $ by the original (resp.~updated, modified) definition.
\fi

\section{Motivation for possible world semantics of causal models}\label{sec: Motivation}
The basic language for causal reasoning described above uses propositional logic as the language of events and for reasoning with causal formulas. However, we are interested in describing causal reasoning in scenarios that involve notions like possibility, knowledge or belief, temporality, uncertainty and accessibility. Here we provide several such examples. 

\begin{example}[Umbrella]\label{ex: umbrella}
Alice is going on a trip to London. She thinks that it may rain when she is there. Thus, she decides to take her umbrella with her for the trip. In this example, the possibility of raining in London in the future seems to be the cause for Alice taking her umbrella with her. 
\end{example}

\begin{example}[Chess]
In a chess game, if knight and the king are only pieces that can move but every king move leads to king getting in check, then the player is forced to move the knight. Suppose that the king can not move to a certain square because it is covered by a bishop. Then it seems reasonable that the fact that bishop covers this square to  be a cause for player being forced to  move the knight. This example shows that reasoning with causality naturally involves considering possibilities.  
\end{example}

%These examples show that future possibilities can be cause of the events in the present. 
%\begin{example}[Career]
%Bob is a student deciding whether or not he should go to college. Bob thinks that he should take college only  if not going to college  can possibly be cause of him not becoming rich in the future. In this case Bob  is reasoning about the possibility of education being cause for becoming rich in future. 
%\end{example}

\begin{example}[Police]
Suppose John is a criminal who is currently absconding. Inspectors Alice and Bob are trying to catch John. John is currently in Amsterdam. He has a train ticket to Brussels. Thus, his (only)  options are to stay in Amsterdam or to take the train to Brussels. Bob decides to go to Brussels to catch John in case he takes the train. John learns this information and decides to stay in Amsterdam where  Alice catches him. In this case, John's belief of Bob's presence in Brussels leads to him staying in Amsterdam. It seems reasonable that this should be part of the cause of John getting caught, even though he was caught by  Alice in Amsterdam. This shows that the knowledge of the agents is crucially involved in causal reasoning.
\end{example}

%\begin{example}[Murder]
%Consider a scenario in which Bob and Alice both shoot at John. Bob's shot misses John but Alice's hits him leading to his death. However, the police investigator on the case does not know which bullet really hit John. In this case it seems that the investigator considers both shootings to be possible causes for John's death but none of them as the certain cause of John's death. This example shows that many reasoning examples involve reasoning about  causality in scenarios with epistemic agents and uncertainty. 
%\end{example}

\begin{example}[Robot]
Consider a scenario in which a robot is being commanded by a scientific team. Upon receiving  command c, the robot completes task t or malfunctions. In this case the possibility of causing  malfunction may become the cause of not sending command c. i.e.~, causal reasoning involves scenarios in which the dependencies between different events may be "indeterministic" or "underdetermined". Halpern considers such scenarios in \cite[Section 2.5]{halpern2016actual},  using the notion of probabilities over causal models. However, in certain cases qualitative reasoning in terms of possibilities may be more appropriate.
\end{example}

\begin{example}[Navigation]
Suppose Alice is trying to reach village A. She reaches a marker which indicates that she is at location B, C or D. She does not know in which of these locations she is at. However, she knows that A is to the east of all of these locations. Thus, she decides to go east. Suppose Alice was actually at point B. The fact that A is to the east of locations C and D is still part of Alice's considerations and  seems to be part  of the cause for her going  east. 
\end{example}

These examples highlight that notions such as possibility, knowledge and uncertainty play an important role in causal reasoning. %In the basic language, these aspects are dealt by adding propositions about each such aspect. For example, in the umbrella example we may have a propositional variable $p$ which stands for 'it is possible that, it may rain in London next week'. However, this leads in many scenarios to large number of propositional variables which do not intuitively seem like primitive events. Thus, causal reasoning with such propositions can be inefficient and unintuitive. 
Possible world semantics, formally described by Kripke models, are the natural logical framework for modeling such  notions. %like uncertainty, possibility, temporality, epistemic knowledge 
In the next section we develop a framework for causal reasoning, based on Kripke frames, which allows for modeling such scenarios in a clear, intuitive and efficient way.

\section{Possible world semantics for causal reasoning}\label{sec: Possible}

\iffalse
\marginnote{
specify if some kind of "empty" intervention (over no variables) is allowed.}
\fi

In this section, we define the causal Kripke model, introduce the modal language for causality and give the corresponding three HP definitions of causality in causal Kripke models. In our framework, we  allow the same variable to possibly take different values in different worlds. Moreover, the structural equations treat the same endogenous variable separately for each different possible world. 
\begin{definition}
  A \textbf{causal Kripke model} is a tuple $\mathcal{K} = (\mathcal{S},W, R, \mathcal{F}$), where $W$ is a finite set of possible worlds, $R \subseteq W \times W$ is an accessibility relation, and 
 $\mathcal{S}=(\mathcal{U}, \mathcal{V}, \mathcal{R})$ is the signature such that   $\mathcal{U}$ and  $\mathcal{V}$ are the disjoint sets of exogenous and endogenous variables, and $\mathcal{R}$ is a function  assigning each   $\Gamma \in \mathcal{U} \cup  \mathcal{V}$ and a world $w\in W$ a set of possible values that $\Gamma$ can take at $w$, and  $ \mathcal{F} = (f_{(X_i,w_j)} \mid X_i \in \mathcal{V}, w_j \in W)$ assigns to each endogenous variable $X_i$ and each world $w_j$ a map such that
  \begin{center}
     $f_{(X_i,w_j)}:\ \mathcal{R}((\mathcal{U} \cup \mathcal{V})\times W)-\{(X_i,w_j)\}) \to \mathcal{R}(X_i,w_j).$
  \end{center}
\end{definition}

\noindent For any causal Kripke model  $\mathcal{K} = (\mathcal{S},W, R, \mathcal{F})$ we refer to $\mathcal{S}$ as its {\em signature}. For any variable $\Gamma$ and world $w$ we use $(\Gamma, w)$ to denote the restriction of variable $\Gamma$ to the world $w$. That is, $(\Gamma, w)$ is a variable which takes a value $c$ iff the propositional variable $\Gamma$ takes the value $c$ at the world $w$.
For any $\Gamma \in  \mathcal{U}$ (resp.~$\Gamma \in  \mathcal{V}$) and any world $w\in W$, we say $(\Gamma,w)$ is an exogenous (resp.~endogenous) variable. Note that we allow the same endogenous variable to have different structural equations associated with it in different worlds. 

\begin{definition}
A  \textbf{ context}  over a causal Kripke model $\mathcal{K} = (\mathcal{S},W,R, \mathcal{F})$ is a function  $\overline{t}$ such that for any $w \in W$, and $U \in \mathcal{U}$, assigns a value in $\mathcal{R}(U,w)$. A \textbf{causal Kripke setting} is a pair $(\mathcal{K},\overline{t})$, where $\mathcal{K}$ is a causal Kripke model and $\overline{t}$ is a context for it.
\end{definition}

\begin{definition}
For any variables $X \in \mathcal{V}$, and $\Gamma \in \mathcal{U} \cup \mathcal{V}$, and any $w,w' \in W$,  we say "$(X,w)$ is a direct cause, or a parent, of $(\Gamma,w')$" if there exist $\gamma,\gamma'  \in \mathcal{R}(\Gamma,w'),$ and $\overline{z} \in \mathcal{R}( (\mathcal{U} \cup \mathcal{V})\times W -\{(\Gamma,w')\}) $ such that 
$f_{(X,w)}(\overline{z},\gamma) \neq f_{(X,w)}(\overline{z},\gamma')$. A causal Kripke model is said to be \textbf{recursive} if it contains no cyclic dependencies. 
\end{definition}

\noindent In recursive models, as there are no cyclic dependencies the values of all endogenous variables at all the worlds are completely determined by the context. If $\mathcal{V}$ only contains binary variables (i.e.~the variable which take values either $0$ or $1$), then for any context $\overline{t}$, and any world $w$,  we use $\overline{t}(w)$ to denote the set of endogenous variables set to value $1$ at $w$ by $\overline{t}$. In this paper, we only consider recursive causal Kripke models.

\begin{definition}
Given a causal Kripke model $\mathcal{K} = (\mathcal{S},W,R, \mathcal{F})$, as {\em assignment} over $\mathcal{K}$  is a function on some subset $ \mathcal{Y} \subseteq \mathcal{V}\times W$ such that, for every $Y =(X,w) \in \mathcal{Y}$, it assigns some value in $\mathcal{R}(X,w)$. 
\end{definition}
\begin{definition}
Let $  \mathcal{K} = (\mathcal{S},W,R, \mathcal{F}) $ be some causal Kripke model  and  $\mathcal{Y}$ be a finite subset of $\mathcal{V} \times W$. Let $\overline{Y} $ be an injective (possibly empty) listing of all the variables in $\mathcal{Y} $. Let $\overline{Y} = \overline{y}$ be an assignment such that for any $Y_i \in \mathcal{Y}$, $ y_i \in \mathcal{R}(Y_i) $. The causal Kripke model obtained from \textbf{intervention} setting values of variables of $\overline{Y}$ to $\overline{y}$ is given by $\mathcal{K}_{\overline{Y} \leftarrow \overline{y}} = (\mathcal{S},W,R, \mathcal{F}_{\overline{Y} \leftarrow \overline{y}}) $ , where $\mathcal{F}_{\overline{Y} \leftarrow \overline{y}}$ is obtained by replacing for every variable $Y_i \in \mathcal{Y} $, the structural equation $f_{Y_i}$ with $Y_i=y_i$.
\end{definition}

\subsection{Modal logic language for describing causality}\label{subse: language}
In this section we define the formal logical framework we introduce for describing causality. Since we want to talk about variables whose values depend on the possible world of a Kripke model, our language will be hybrid in character, augmenting the standard language (as presented e.g.\ in Section \ref{ssec:basic language}) not only with modal operators but also with a countable set of names, denoted with $W$. In principle each model $M$ comes with an assignment from $W$ to points of $M$, however in practice we will often conflate names with elements of Kripke models. The reason we require a countable number of names, even though the models are always finite, is because there is no bound on the size of the models.  We will denote the language with $\mathrm{L}_M(W)$. We often omit $W$ and write $\mathrm{L}_M$ when $W$ is clear from the context. $\mathcal{S}$ is a given signature, and all $X, Y, x, y$ come from $\mathcal{S}$. In what follows we will consistently use $Y$ to denote a variable parametrized with a name for a world (i.e.\ $Y=(X,w)$). It is important to notice that in our language interventions involve only such variables. Any  event $\alpha$ and formula $\phi$ of the language  $\mathrm{L}_M$ is defined by the following recursion. 
\begin{center}
    $\alpha ::= X=x\mid (X,w)=x \mid \neg \alpha \mid \alpha \wedge \alpha \mid \Box \alpha  \quad \text{where}, \, X \in \mathcal{V}, w \in W$

$\phi::=  X=x  \mid (X,w)=x \mid\neg \phi \mid \phi \wedge \phi \mid \Box\phi\mid [\overline{Y}\leftarrow\overline{y}]\alpha \quad \text{where}, \, \overline{Y}\leftarrow\overline{y} \, \text{is an intervention}$

\end{center}

In particular, the language $\mathrm{L}_M$ has two types of atomic propositions, using variables of the form $X$ and of the form $(X,w)$. The second, the hybrid aspect of our language, provides global information regarding the Kripke model. For any causal Kripke setting $(\mathcal{K},\overline{t})$ with 
$\mathcal{K} = (\mathcal{S},W,R, \mathcal{F})$,  any causal formula $\phi$, and any world $w \in W$, we define satisfaction relation $\Vdash$ in the following way. For any primitive event $X=x$ (resp.~$(X,w')=x$),  $(\mathcal{K},\overline{t},w) \Vdash X=x$ 
(resp.~$(\mathcal{K},\overline{t},w) \Vdash (X,w')=x$) iff the value of $X$ is set to be $x$ at $w$ (resp.~at $w'$) by the context $\overline{t}$. Note that the satisfaction of $ (X,w')=x$ is independent of the world it is evaluated at. The satisfaction relation for Boolean connectives is defined by standard recursion. For the $\Box$ operator, 
\begin{center}
 $(\mathcal{K},\overline{t}, w) \Vdash \Box\alpha \quad \text{iff} \quad \text{for all}\, w', \, w R w' \, \text{implies}\, (\mathcal{K},\overline{t}, w') \Vdash \alpha.$   
\end{center}
Let $\mathcal{Y} \subseteq \mathcal{V} \times W$ be a set of endogenous variables. Satisfaction of intervention formulas is defined as for any  event $\alpha$,  $(\mathcal{K}  ,\overline{t}, w) \Vdash [\overline{Y}\leftarrow\overline{y}]\alpha $ iff  $(\mathcal{K} _{[\overline{Y}\leftarrow\overline{y}]},\overline{t}, w) \Vdash \alpha $.  Satisfaction for Boolean combinations of causal formulas is defined in a standard manner. For the $\Box$ operator, 
\begin{center}
    $(\mathcal{K},\overline{t}, w) \Vdash \Box\phi \quad \text{iff} \quad \text{for all}\, w', \, w R w' \, \text{implies}\, (\mathcal{K},\overline{t}, w') \Vdash \phi.$
\end{center}
We now extend the HP definition(s) of causality to the setting of causal Kripke models.
\begin{definition}\label{def: causality original}
Let $\alpha$ be any event.   For $\overline{Y}\subseteq \mathcal{V} \times W, \overline{Y} = \overline{y} $ is an {\em actual cause}  of $\alpha$ in a  causal Kripke setting $(\mathcal{K},\overline{t})$ at a world $w$ if the following conditions hold.
\begin{itemize}
    \item[AC1.] 
$(\mathcal{K},\overline{t},w) \Vdash \alpha$\ 
 and for every $w_j \in W$, $(\mathcal{K},\overline{t},w_j)\Vdash (X_i,w_j)=y_{ij}$, for every $(X_i, w_j)=y_{ij} \in \overline{Y}=\overline{y}  $. 
    \item[AC2a.] There exists a partition of $\mathcal{V}\times W$ into two disjoint subsets  $\overline{Z}$ and  $\overline{N}$ with $ \overline{Y} \subseteq  \overline{Z}$ and  settings $\overline{y'}$ and $\overline{n}$ of variables in $\overline{Y}$ and $\overline{N}$, such that 
    \begin{center}
      $(\mathcal{K},\overline{t},w) \Vdash [\overline{Y} \leftarrow \overline{y'}, \overline{N} \leftarrow \overline{n} ] \neg \alpha.$  
    \end{center}

     \item[AC2$b^o$.] Let $\overline{z}^\ast$ be the unique setting of the variables in $\overline{Z}$ such that $(\mathcal{K},\overline{t},w)\Vdash\overline{Z} = \overline{z}^\ast$. If $(\mathcal{K},\overline{t},w')\Vdash X=z^\ast$, for every $(X,w') =z^\ast \in \overline{Z}=\overline{z}^\ast$, then for all subsets $\overline{Z}'$ of $\overline{Z} \setminus \overline{Y}$ we have 
     \begin{center}
          $(\mathcal{K},\overline{t},w) \Vdash [\overline{Y} \leftarrow \overline{y}, \overline{N} \leftarrow \overline{n}, \overline{Z}' \leftarrow \overline{z}'^\ast ]  \alpha.$      \footnote{Here we use the abuse of notation that if $\overline{Z}'\subseteq\overline{Z}$ and $\overline{Z}=\overline{z}^\ast$, then $\overline{z}'^\ast$ in $\overline{Z}'\leftarrow\overline{z}'^\ast$ refers to the restriction of $\overline{z}^\ast$ to $\overline{Z}'$.}
     \end{center} 
     
     \item[AC3.] $\overline{Y}$ is a minimal set of variables that satisfy AC1 and AC2.
\end{itemize}
We say that $\overline{Y} = \overline{y} $ is an \textbf{actual cause}  of $\alpha$ in  a causal Kripke  setting $(\mathcal{K},\overline{t})$ at  a world $w$ by \textbf{updated definition} iff AC1, AC2a, AC3 hold and AC2$b^o$ is replaced by the following condition. 
\item[AC2$b^u$.] Let $\overline{z}^\ast$ be the unique setting of the variables in $\overline{Z}$ such that $(\mathcal{K},\overline{t},w)\Vdash\overline{Z} = \overline{z}^\ast$. If 
     $(\mathcal{K},\overline{t},w')\Vdash X=z^\ast$, for every $(X,w')=z^\ast \in \overline{Z} = \overline{z}^\ast$, then for all subsets $\overline{Z}'$ of $\overline{Z} \setminus \overline{Y}$ and $\overline{N}'$ of $\overline{N}$ we have  
\begin{center}
         $(\mathcal{K},\overline{t},w) \Vdash [\overline{Y} \leftarrow \overline{y}, \overline{N}' \leftarrow \overline{n}, \overline{Z}' \leftarrow \overline{z}'^\ast ]  \alpha.$
\end{center}
We say that $\overline{Y} = \overline{y} $ is an \textbf{actual cause}  of $\alpha$ in  a causal Kripke setting $(\mathcal{K},\overline{t})$ at a world $w$ by \textbf{modified definition} iff AC1, AC3 hold and AC2 is replaced by the following condition.
\item[$AC2a^m$.] If there exists a set of variables $\overline{N} \subseteq \mathcal{V} \times W$, and  a setting  $\overline{y}'$ of the variables in $\overline{Y}$ such that if $\overline{n}^\ast $ is such that 
     $(\mathcal{K},\overline{t},w')\Vdash X=n^\ast$, for every $(X,w')=n^\ast \in \overline{N}=\overline{n}^\ast  $, then 
     \begin{center}
            $(\mathcal{K},\overline{t},w) \Vdash [\overline{Y} \leftarrow \overline{y}', \overline{N} \leftarrow \overline{n}^\ast ]  \neg\alpha.$
     \end{center}  
\end{definition}

We will refer to these definitions as original, updated, and modified  definitions  henceforth. Example \ref{ex:stalemate revisited} shows that these definitions do not in general coincide.  Theorem  \ref{thm:reln between three defn} which relates these definitions in causal models  can be generalized to causal Kripke models in a straightforward manner (see Theorem, \ref{thm:reln between three defn modal}).

For any set of variables $\mathcal{Y} \subseteq \mathcal{V} \times W$,  we use $cause^o (\overline{Y}=\overline{y}, \alpha) $ (resp.~$cause^u (\overline{Y}=\overline{y},\alpha)$, $cause^m (\overline{Y}=\overline{y},\alpha)$) as an  abbreviation for stating $\overline{Y}=\overline{y}$ is a cause of $\alpha$ by the original (resp.~updated, modified) definition.
We write $ (\mathcal{K},\overline{t},w) \Vdash cause^o(\overline{Y}=\overline{y}, \alpha)$ (resp.~$ (\mathcal{K},\overline{t},w) \Vdash cause^u(\overline{Y}=\overline{y}, \alpha)$, $ (\mathcal{K},\overline{t},w) \Vdash cause^m (\overline{Y}=\overline{y},\alpha)$ ) as an abbreviation for stating  $\overline{Y}=\overline{y}$ is a cause of $\alpha$ in causal Kripke setting $ (\mathcal{K},\overline{t}) $ at a world $w$ by the original (resp.~updated, modified) definition. Moreover, For $x=o,u,m$, we write   $ (\mathcal{K},\overline{t},w) \Vdash \Box cause^{x} (\overline{Y}=\overline{y}, \alpha)$  if for all $w'$ such that $w R w'$, $ (\mathcal{K},\overline{t},w') \Vdash cause^x(\overline{Y}=\overline{y}, \alpha)$ and $ (\mathcal{K},\overline{t},w) \Vdash \Diamond cause^x(\overline{Y}=\overline{y}, \alpha)$  if there exists  $w'$ such that $wRw'$ and  $ (\mathcal{K},\overline{t},w') \Vdash cause^x(\overline{Y}=\overline{y}, \alpha)$.

\subsection{Axiomatization}\label{ssec:axiomatization}
In \cite{halpern2016actual}, Halpern provides a sound and complete axiomatization for the logic of causality. This axiomatization can be extended to the modal  logic of causality by adding the following axioms to the axiomatization in \cite[Section 5.4]{halpern2016actual}:
\begin{itemize}
    \item All substitution instances of axioms of basic modal logic $\mathrm{K}$. 
    \item Necessitation rule:  from $\phi$ infer $\Box \phi$
    \item $\Diamond$-axiom : $[\overline{Y} \leftarrow \overline{y}] \Diamond \phi \Leftrightarrow  \Diamond 
 [\overline{Y} \leftarrow \overline{y}]\phi$ \quad and \quad $\Box$-axiom : $[\overline{Y} \leftarrow \overline{y}] \Box \phi \Leftrightarrow  \Box 
 [\overline{Y} \leftarrow \overline{y}]\phi$
 \item G-axiom: $([\overline{Y} \leftarrow \overline{y}](X,w)=x) \Rightarrow \Box([\overline{Y} \leftarrow \overline{y}](X,w)=x )$
\end{itemize}
%\marginnote{about the soundness of G-axiom, the point is that the formula $[\overline{Y} \leftarrow \overline{y}](X,w)=x)$ is world-independent (due to the hybrid features of the language), and need to say it contains the case of empty intervention.}

Notice that, similar to the axiomatization in  \cite[Section 5.4]{halpern2016actual}, the schemes $\Diamond$-axiom, $\Box$-axiom, and G-axiom  include empty interventions. When importing the axioms from \cite[Section 5.4]{halpern2016actual} axiom scheme C5 involves only variables of the form $(X,w)$. For the axiom schemes C1-4 and C6, the axioms involve atoms both of the form $X=x$ and of the form $(X,w)=x$. Notice also that G-axiom is similar to axioms in Hybrid modal logic. 

Since the language in \cite{halpern2016actual} is finite (modulo classical tautologies), weak and strong completeness coincide. However our language is countable (given that $W$ is countable). Since there is no upper bound on the size of the models, we cannot hope to have strong completeness w.r.t.\ finite models. However the axioms presented in this section are sound and weakly complete. In Appendix  \ref{sec: compl}  we provide the proofs of \textbf{soundness} and weak \textbf{completeness} w.r.t.~the modal logic of causality.

%\textcolor{red}{why do you stop short of a strong completeness result?
%10.1007/978-3-642-35873-9-You should also say that what you produce is a family of axiom systems, one for each (finite) signature.}
%\begin{theorem} \label{thm: completeness modal}
%The axiomatization  above is sound and complete w.r.t.~the modal logic of causality.
%\end{theorem}
%\begin{proof}
%See Appendix \ref{sec: compl}. 
%\end{proof}

\section{Examples and applications}\label{sec:Examples}
In this section, we analyze the examples discussed in Section \ref{sec: Motivation} using  causal Kripke models. For any endogenous variable $X$, and any world $w$, we use $Eq(X,w)$ to denote the structural equation for $X$ at $w$. Throughout this section, we use $U$ to denote exogenous variables only. 
\begin{example}[Umbrella]  Let $\mathcal{S}$ be the signature with  endogenous variables $p$, $q$, and $r$ standing for `it rains in London' and `Alice adds  umbrella to the luggage' and  `Alice is in London'.  Let $w_0$ be the current world and $w_1, w_2, w_3$ be the future possible worlds considered by Alice. Let $W =\{w_0,w_1,w_2, w_3\}$ and  $R= \{(w_0,w_1), (w_0,w_2),(w_0,w_3)\}$. Let $U =(U_1,U_2)\in \{0,1\}^2$ be  such that $(p,w)=(U_1,w)$ and  $(r,w)=(U_2,w)$ for  any $w \in W$.  Let   $Eq(q,w) = \Diamond (p \wedge r)$ for all $w$, i.e., Alice puts her umbrella in her luggage  if she thinks it is possible that in the future she will be in London and it rains there.

Let $\overline{t}$ be a context such that $U$ is set to be $(0,0)$,$(0,1)$, $(1,0)$ and $(1,1)$ at the worlds  $w_0, w_1, w_2$ and $w_3$ respectively. We have $\overline{t} (w_0)= \{q\}$, $\overline{t} (w_1)= \{r\}$, $\overline{t} (w_2)= \{ p\}$, $\overline{t} (w_3)= \{ p,r\}$. Here,  $(p,w_3)=1$ and $(r,w_3)=1$ are both causes of $q=1$ at $w_0$ by all three  definitions. 

We show  that $(p,w_3)=1$ is a cause by the original and updated definitions. The proof for  $(r,w_3)=1$ is analogous.
Indeed, as $(\mathcal{K},\overline{t},w_3)\Vdash (p,w_3)=1$, $(\mathcal{K},\overline{t},w_3)\Vdash (r,w_3)=1$ 
and $(\mathcal{K},\overline{t},w_0) \Vdash q$, AC1 is satisfied. Let $\overline{Z} = \{(r,w_3),(p,w_3)\}$ and $\overline{N} = \varnothing$, $\overline{Y} =\{(p,w_3)\} $, and $\overline{y'} =0$.  Then from the structural equation as no  world related to $w_0$ satisfies $p \wedge r$ under this intervention, we have 
\begin{center}
    $(\mathcal{K},\overline{t},w_0) \Vdash [\overline{Y} \leftarrow 0] \neg (q=1) \quad \text{and} \quad  (\mathcal{K},\overline{t},w_0) \Vdash [\overline{Y} \leftarrow 1, \overline{Z'} \leftarrow \overline{z}^\ast ]  q=1.$
\end{center}

\noindent where $\overline{Z'}=(r, w_3)$, $\overline{z}^\ast=(1,1)$ as described by the context. 
Thus, AC2 is satisfied and AC3 is trivial as we are considering a single variable. The updated definition in this case is equivalent to the original definition as $\overline{N} =\emptyset$. The modified definition is satisfied for the same setting $\overline{y'} =0 $ and $\overline{N} = \{(r,w_3)\}$. 
Thus, the fact that it rains in the world $w_3$ is a cause of Alice carrying her umbrella by all three definitions. 
\end{example}

Now consider a slight variation of the above example in which we are sure Alice will be in the London and do not include $r$ as a variable in our analysis. In this case, the structural equation for $q$ is given by $q=\Diamond p$. Note that, in  this new model, we can argue in the way similar to the above example that any world $w'$ accessible from $w_{0}$,  $(p,w')=1$ would be a part of the cause of Alice carrying her umbrella at $w_0$ by all three definitions. This can be interpreted as the fact that 'Alice considers the possibility of a future world in which it rains in London and she will be in London' is a part of cause of  her adding umbrella to luggage. 

In general, for any event $\alpha$ and endogenous variable $X$ we say that "the possibility of ${X}={x}$" is a  cause of $\alpha$ at $w$  iff   $\bigwedge \{(X,w')=x \mid w R w' \& (\mathcal{K},\overline{t}, w') \Vdash {X}={x}\}$ is a cause of  $\alpha$ at $w$ by the modified definition. Here, we use the modified definition because as mentioned by Halpern~\cite[Example 2.3.1]{halpern2016actual}, the conjunction being a cause by modified definition can in fact be interpreted as a cause being disjunctive, i.e.~, the disjunction of the conjuncts can be interpreted as the cause of the event. Hence under this interpretation the existence of some $w'$ which is accessible from $w_{0}$ and where ${X}={x}$ is a cause of $\alpha$ here.  This can be interpreted as "the possibility of ${X}={x}$" being a cause of $\alpha$. Thus, in the variation of the example discussed in the above paragraph,  we can say that the possibility of the world where it rains in London is a cause of Alice adding umbrella to her luggage. 

\begin{example}[stalemate] \label{ex:stalemate}
 Let $\mathcal{S}$ be the signature with  endogenous variables $p$,  $q$ and $r$ standing for `The king is in check', `The king and the knight are the only  pieces that can move in the current position' and   `The player is forced to move the knight'. Let $w_0$ be the current  position. Let $w_1$ and $w_2$ be the positions obtained from  the (only) available moves by the king. Let $W=\{w_0, w_1, w_2\}$ and $R= \{(w_0,w_1), (w_0,w_2)\} $.
Let $U  =(U_1,U_2)\in \{0,1\}^2$ be such that  $(p,w)=(U_1,w)$ and  $(q,w)=(U_2,w)$ for  any $w \in W$. Let   $Eq(r,w)= \neg p \wedge q \wedge \Box p$ at any $w$, i.e., the player is forced to move the knight if the king and the knight are the only pieces that can move, the king  is not in check,  and the king's every available move leads to the king being in check. 

Let $\overline{t}$ be a context such that $U$ is set to be $(0,1)$,$(1,1)$, and  $(1,0)$ at the worlds  $w_0$, $w_1$, and $w_2$ respectively. We have $\overline{t} (w_0)= \{q,r\}$, $\overline{t} (w_1)= \{p,q\}$, $\overline{t} (w_2)= \{p\}$. In the same way as in the last example,  we can show that $(p,w_0)=0$, $(q,w_0)=1$,   $(p,w_1)=1$  and  $(p,w_2)=1$  are  all the  causes of $r=1$ at $w_0$ by all three  definitions. This can intuitively be seen as the certainty of  the king ending up in a check regardless of the king move (while not currently being in check) is a part of cause of being forced to move the knight. 

{\em In general, for any variable $X$, and  any event $\alpha$ we say that the certainty of ${X}={x}$ is a  cause of $\alpha$ at $w$  
iff  $ ({X},w')={x}$ is a cause of  $\alpha$ at $w$ for all $w R w'$ by  the modified definition.}
\end{example}

\begin{example}[Police]
 Let $\mathcal{S}$ be the signature with  endogenous variables $p$, $q$, $r$, and $s$ standing for `Inspector Bob is in Brussels', `Inspector Alice is in Amsterdam',  `John takes the train', and  `John is caught in Amsterdam'.  Let $w_0$ be the current world and $w_1, w_2$ be the  possible future worlds considered by John.  Let $W =\{w_0, w_1,w_2\}$ and $R= \{(w_0,w_1), (w_0,w_2)\} $.

Let $U \in \{0,1\}^2$ be  such that $(p,w)=(U_1,w)$ and  $(q,w)=(U_2,w)$ for  any $w \in W$.  Let $Eq (r,w_0)= \neg \Box p$ and $Eq (r,w_1)=Eq (r,w_2)= 1$.  i.e.~, John takes the train if there is a  possible future scenario in which inspector Bob is not in Brussels (John considers future scenarios when he takes the train).  Let  $Eq (s,w_0)= q \wedge \neg r$ (John gets caught in Amsterdam if Alice is there and he does not take the train) and $Eq(s,w_1)=Eq (s, w_2)=0$ (John considers future scenarios in which he is not caught).

Let $\overline{t}$ be a context such that $U$ is set to be $(1,1)$ at  all the worlds. We have $\overline{t} (w_0)= \{p,q,s\}$ and $ \overline{t} (w_1)= \overline{t} (w_2) = \{p,q,r\}$. It is easy to check that  $(p,w_1)=1$ and $(p,w_2)=1$ are both causes of $s=1$ at $w_0$ by all three  definitions. Thus, the fact that  Bob is present in Brussels in all possible worlds considered by John is a cause of John getting caught in Amsterdam. If we assume that John knows Bob is in Brussels iff he is actually in Brussels, then we can say that Bob's presence in Brussels is a cause of John getting caught in Amsterdam. 

Here we have assumed that John's knowledge of  Bob's presence in Brussels is the same as  Bob being actually present in Brussels. However, this need not be the case always. %Consider the scenario where there exists another possible  world $w_3$ with $p=0$, and i.e.~ in this world inspector Bob is not in Brussels. However, if John does not consider this scenario as possible (i.e. if $\neg (w_0  R w_3)$), then John will still not take the train. 
Now consider a slightly more complicated version of the same story in which we have another endogenous variable $o$ standing for `John lost his ticket'. Suppose John does not take the train if he loses the ticket or he knows inspector Bob is in Brussels. i.e.~,$Eq(r,w_0)= \neg (\Box p \vee o)$. Other structural equations remain the same. Let $\overline{t}'$ be a context that sets variables $p$, $q$ to the same values as $\overline{t}$ at  all the worlds and sets $o$ to be $1$ at $w_0$. In this case $(p,w_1)=1$, $(p,w_2)=1$ and $(o,w_0)=1$ are all the causes of  $s=1$ at $w_0$ by all three  definitions. Now consider slightly different context $\overline{t}''$ such that $p$ is set to be true at worlds $w_0$ and $w_1$, $q$ at worlds $w_0$, $w_1$, and $w_2$ and $o$ at world $w_0$. In this case, we again have $r=0$ and $s=1$ at $w_0$. However, only $(o,w_0)=1$ is a cause for $s=1$ at $w_0$ by all three definitions. Now suppose that  Bob actually did go to Brussels, however John does not \textbf{know} this information  and thinks that there is a possibility that  Bob may not be in Brussels. The only reason he does not take the train is that he lost the ticket. Thus,  Bob being present in Brussels is not a cause of John getting caught in Amsterdam in this case. This  shows that the knowledge John has about the presence of  Bob in Brussels (and not just presence itself) is an important part of causal reasoning. 
\end{example}

\begin{example}[Robot]
 Let $\mathcal{S}$ be the signature with endogenous variables $p$,  $q$ and $r$ standing for `The command $c$ is sent by the scientific team',  `The task $t$ is completed by the robot' and   `The robot malfunctions'.  Let $w_0$ be the current world, the world in which the scientific team is reasoning. Let $w_1$ and $w_2$ be the possible worlds considered by the team.  Let $W= \{w_0,w_1, w_2\}$ and  $R= \{(w_0,w_1), (w_0,w_2)\} $.

Suppose $Eq(q,w_1)= \blacklozenge p$, $Eq(r,w_1)=0$, $Eq(r,w_2)= \blacklozenge p$, $Eq(q,w_2)=0$, and  $Eq (q,w_0)=Eq(r,w_0)=0$, where $\blacklozenge$ is the diamond operator corresponding to the relation $R^{-1}$, i.e., there are two possible scenarios. In one scenario sending command $c$ leads to the completion of task $t$ and no malfunctioning, while in the other it leads to the malfunctioning of the robot and task $t$ is not completed. 

Let $U  \in \{0,1\}$ be  such that $(p,w_i)=(U,w_i)$.   Let $\overline{t}$ be a context such that $U$ is set to be $1$ in all the worlds.  We have $\overline{t} (w_0)= \{p\}$, $\overline{t} (w_1)= \{p,q\}$, $\overline{t} (w_2)= \{p,r\}$.  It is easy to see that $(p,w_0)=1$ is the cause of $r=1$ at $w_2$ but not at $w_1$ (it is not even true at $w_1$) by all three definitions. Suppose the scientific team believes that if  sending command can possibly cause malfunctioning then command shouldn't be sent. Then as $ (\mathcal{K},\overline{t},w_0) \Vdash  \Diamond cause ((p,w_0)=1, r=1)$ ( for all three definitions), the team will decide not to send the command. On the other hand, if the team believes that the command  should be sent if it can possibly cause  the completion of the task,  then it must be sent as  $(\mathcal{K},\overline{t},w_0) \Vdash  \Diamond cause ((p,w_0)=1, q=1)$, i.e.~ sending signal can cause completion of task $t$. 
\end{example}

\begin{example}[Navigation]
 Let $\mathcal{S}$ be the signature with endogenous variables $p_x$,  $q$ and $r$ standing for `The current location of Alice is x' for $x=B,C,D$,  `Point $A$ is to the east of Alice's current location' and   `Alice moves to the east'. Alice  does not know if she is at point $B$, $C$ or $D$. Let $w_1$, $w_2$ and $w_3$ be  possible worlds and  $U \in \{0,1\}^4$ be  such that $(p_B,w)=(U_1,w)$, $(p_C,w)=(U_2,w)$, $(p_D,w)=(U_3,w)$  and  $(q,w)=(U_4,w)$ for  any $w \in W$. 

Let $\overline{t}$ be a context such that $U$ is set to be $(1,0,0,1)$, $(0,1,0,1)$, and  $(0,0,1,1)$  at worlds $w_1$, $w_2$ and $w_3$.  We  have $\overline{t} (w_1)= \{p_B,q,r\}$,  $\overline{t} (w_2)= \{p_C,q,r\}$, and  $\overline{t} (w_3)= \{p_D,q,r\}$.  Worlds $w_1$,  $w_2$ and $w_3$ here represent the possible scenarios considered by Alice. With the current available knowledge these worlds are indistinguishable from each other for Alice. This can be represented by $R= W \times W$. At any world $w$  $Eq(r,w)=\Box q$, i.e.~Alice moves to the East iff she \textbf{knows} point $A$ is to the East of her current location.  Here, in the world $w_1$ in which Alice is at $B$ (the real situation), $(q,w_1)=1$,  $(q,w_2)=1$ and  $(q,w_3)=1$ are all causes of $(r,w_1)=1$ by all three definitions. Thus, the fact that $A$ is to the East of point $C$ or point $D$ is also a cause of Alice moving to East even if she is not present there. 
\end{example}
These examples show that causal Kripke models can be used to model several different scenarios involving causality interacting with notions like possibility, knowledge and uncertainty.

\subsection{Sufficient causality} \label{ssec:sufficint causality}
Halpern discusses the notion of sufficient causality in \cite{halpern2016actual} to model the fact that people's reasoning about causality depends on how sensitive the causality ascription is to changes in various other factors.  ``The key intuition behind the definition of sufficient causality is that not only does $\overline{X}=\overline{x} $
suffice to bring about $\phi$ in the actual context, %(which is the intuition that $AC2(b^o)$ and $AC2(b^u)$ are trying to capture) 
but it also brings it about in other “nearby” contexts. Since the framework does not provide a metric on contexts, there is no obvious way to define nearby context. Thus, in the formal definition below, I start by considering all contexts.''\cite[Section 2.6]{halpern2016actual}   Sufficient causality is thus defined in \cite{halpern2016actual} using Definition \ref{def:sufficient causality Halpern}. 

We can use the framework of causal Kripke models to define sufficient causality  for a causal setting $(M, u)$ in terms of nearby contexts instead of all the contexts (as suggested by Halpern) in the following way. We consider the causal Kripke model $\mathcal{K} = (\mathcal{S},W,R, \mathcal{F}) $, where $\mathcal{S}$ is the signature of  $M$, $W$ is the set of all the possible contexts on $M$, and  $\mathcal{F}$ is the set of structural equations such that for any structural equation for the endogenous variable $X$, $Eq(X,w)$ is the same as the structural equation for $X$ in $M$ and the  relation $R \subseteq W \times W$ is such that  $u R u'$ iff context $u'$ is nearby $u$.  Let $\overline{t}$ be the setting of exogenous variables so that for any possible world the endogenous variables are set by the context identifying that world. 
Let $\overline{X} =\overline{x}$ be as in Definition \ref{def:sufficient causality Halpern} and let $\overline{Y} = \overline{X} \times W$. For any   $Y=(X,u)$, $Y=y$ iff  $X$ is set to be $x$ by the context ${u}$. Let   $\overline{Y} \leftarrow \overline{y}$ is intervention setting $X$ to $x$ in  all the possible contexts.  In this structure we can describe sufficient causality in terms of nearby contexts by replacing the clause $SC3$ in the Definition \ref{def:sufficient causality Halpern} by the condition 
\begin{center}
  $ u R u' \quad \implies (\mathcal{K}, \overline{t}, u') \models  [\overline{Y} \leftarrow \overline{y}]  \alpha \quad \text{or equivalently} \quad   (\mathcal{K}, \overline{t}, u) \models \Box  [\overline{Y} \leftarrow \overline{y}]  \alpha. $  
\end{center}
 We call this property $SC3$-local as we  only  require that the intervention  $[\overline{Y} \leftarrow \overline{y}]$ makes $\alpha$ true in nearby (not all) contexts. 
 
In this Section, we have mainly only considered the causal Kripke model with only one relation which is the "nearby" relation. However, we can also consider causal Kripke models with multiple relations in which we have a nearby relation $N$ on the worlds in addition to the other accessibility relations denoting relationships between world like time, indistingushibility, etc. The definition of sufficient  causality discussed above  can naturally be extended to this setting allowing us to describe the sufficient causality in the setting of causal Kripke models. Here,  we do not go into details of this generalization but we believe this would be an interesting direction for future research.

%Here, we do not go into the discussion on which contexts are defined to be nearby as it is outside the scope of this work. However, our model allows to formally describe sufficient causality using simple formulas for any such definition.

\section{Conclusions and future directions}\label{sec:conclusion}
In this paper, we have developed a possible world semantics for reasoning about actual causality. We develop a modal language and logic  to formally reason in this framework. This language is used to generalize the HP definitions of actual causality for this framework. We provide a sound and complete axiomatization of the modal logic of causality developed, and give a number of examples to illustrate how our model can be used to reason about causality. Finally, we show that our  framework allows us to define the intended notion of sufficient causality in a straightforward and more intuitive manner. 

This work can be extended in several directions. First,  results regarding the relationship of the HP definitions with but-for causality \cite[Proposition 2.2.2]{halpern2016actual},  and  transitivity of cause \cite[Section 2.4]{halpern2016actual}, can be generalized to our  modal setting.  Secondly, we can allow for interventions on the relation $R$ in causal Kripke models. Indeed, in many scenarios intuitively the cause for some event is accessibility to some possible world. Allowing interventions on $R$ would allow us to model such scenarios. Finally, similar to sufficient causality, other notions related to actual causality like normality (or typicality) and graded causation can be described in more nuanced and possibly multiple ways  using the causal modal language. 

\bibliographystyle{eptcs}
\bibliography{generic}
\appendix
\section{HP definitions of causality and sufficient causality} \label{ssec:HP definitions basic}
\begin{definition}[\cite{halpern2016actual}, Definition 2.2.1]\label{def: causality original}
Let $\alpha$ be an event obtained by Boolean combination of primitive events. Let $\mathcal{X} \subseteq \mathcal{V}$ be a set of endogenous variables.  $\overline{X} = \overline{x} $ is an \textbf{actual cause}  of $\alpha$ in a causal setting $(M, \overline{t})$ if the following conditions hold.
\begin{itemize}
    \item[AC1.] $(M, \overline{t}) \Vdash \overline{X} = \overline{x}$ and $(M, \overline{t}) \Vdash \alpha$. 
    \item[AC2a.] There is a partition of $\mathcal{V}$ into two disjoint subsets  $\overline{Z}$ and  $\mathcal{W}$ with $ \overline{X} \subseteq  \overline{Z}$ and a setting $\overline{x}'$ and $\overline{w}$ of variables in $\mathcal{X}$ and $\mathcal{W}$, respectively, such that 
    \[
    (M, \overline{t}) \Vdash_p [\overline{X} \leftarrow \overline{x}', \overline{W} \leftarrow \overline{w} ] \neg \alpha.
    \]
     \item[AC2$b^o$.] If $\overline{z}^\ast $  is such that $(M, \overline{t}) \Vdash \overline{Z}=\overline{z}^\ast $, then for all subsets $\overline{Z}'$ of $\overline{Z} \setminus \mathcal{X}$ we have 
     \[
     (M, \overline{t}) \Vdash_p [\overline{X} \leftarrow \overline{x}, \overline{W} \leftarrow \overline{w}, \overline{Z}' \leftarrow \overline{z}^\ast ]  \alpha.
     \]
     \item[AC3.] $\mathcal{X}$ is minimal set of variable that satisfy AC1 and AC2.
\end{itemize}
We say that $\overline{X} = \overline{x} $ is an {\em actual cause}  of $\alpha$ in  a causal setting $(M, \overline{t})$ by {\em updated definition} iff AC1, AC2a, AC3 hold and AC2$b^o$ is replaced by the following condition.
\item[AC2$b^u$.] If $\overline{z}^\ast $ is such that $(M, \overline{t}) \Vdash \overline{Z}=\overline{z}^\ast $, then for all subsets $\overline{Z}'$   of $\overline{Z} \setminus \mathcal{X}$ and $\mathcal{W}'$   of $\mathcal{W}$  we have 
     \[
     (M, \overline{t}) \Vdash_p [\overline{X} \leftarrow \overline{x}, \overline{W}' \leftarrow \overline{w}, \overline{Z}' \leftarrow \overline{z}^\ast ]  \alpha.
     \]
We say that $\overline{X} = \overline{x} $ is an {\em actual cause}  of $\alpha$ in  a causal setting $(M, \overline{t})$ by {\em modified definition} iff AC1, AC3 hold and AC2 is replaced by the following condition.
\item[$AC2a^m$.] If there exists a set of variables $\mathcal{W} \subseteq \mathcal{V}$, and  a setting  $\overline{x}'$ of variable in $\overline{X}$ such that if $ (M, \overline{t}) \Vdash \overline{W} = \overline{w}^\ast   $, then 
     \[
     (M, \overline{t}) \Vdash_p [\overline{X} \leftarrow \overline{x}', \overline{W} \leftarrow \overline{w}^\ast ]  \neg\alpha.
     \]
\end{definition}
The following theorem describes relationship between these three definitions of causality.

\begin{definition}[\cite{halpern2016actual}, Definition 2.6.1]\label{def:sufficient causality Halpern}
 $\overline{X}=\overline{x}$ is a sufficient cause of $\alpha$ in the causal setting $(M, \overline{u})$ if the following
conditions hold:
\begin{itemize}
    \item [SC1.]  $(M, \overline{u}) \models\overline{X}=\overline{x}$ and $(M, \overline{u}) \models \alpha$. 
    \item [SC2.]  Some conjunct of $\overline{X}=\overline{x}$ is part of a cause of $\alpha$  in  $(M, \overline{u})$. More precisely, there exists a conjunct $X=x$ of $\overline{X}=\overline{x}$  and another (possibly empty) conjunction $\overline{Y}=\overline{y}$ such that $X =x \wedge  \overline{Y}=\overline{y} $ is a cause of $\alpha$ in  $(M, \overline{u})$; i.e.~, AC1, AC2, and AC3 hold for (possibly empty) conjunction $\overline{Y}=\overline{y}$ such that $X =x \wedge  \overline{Y}=\overline{y} $ 
     \item [SC3.] $(M, \overline{u}') \models  [\overline{X} \leftarrow \overline{x}]  \alpha$ for all contexts $\overline{u}'$. 
     \item $\overline{X}$ is the minimal set satisfying above properties. 
\end{itemize}
\end{definition}
 \section{Relationship between three  HP definitions of causality}
 In \cite{halpern2016actual} Halpern gives examples to show that the  three HP definitions do not coincide with each other. Here we give one example to show that  the  modified definition may not coincide with the original and updated definition in the causal Kripke model. Similar example can be given to show that original and updated definition do not coincide. 
 
 \begin{example}[stalemate detailed]\label{ex:stalemate revisited}
We consider the following variation of the example \ref{ex:stalemate}. Let $\mathcal{S}$  be signature with endogenous variables $p_1$ and $p_2$ instead of $p$ (keeping the other  endogenous  variables unchanged) standing for 'The king is in check by the opponent's queen' and 'The king is in check by the opponent's king'. Let $U =(U_1,U_2,U_3)  \in \{0,1\}^3$ be  such that  $(p_1,w)=(U_1,w)$, $(p_2,w)=(U_2,w)$, and  $(q,w)=(U_3,w)$ for  any $w \in W$.  Let the structural equation for $r$ at $w_0$ be given by $r = \neg (p_1 \vee p_2) \wedge q \wedge \Box (p_1 \vee p_2)$. i.e.~, the player is forced to move the knoight if  if  the only pieces that can move are the king and the knight, ther king  is not in check  and every possible king move leads to king being in the  check  by the king or the queen (We assume there are no other pieces on the board) of the opponenet . Let $\overline{t}$ be a context such that $U$ is set to be $(0,0,1)$,$(1,1,1)$, and  $(0,1,0)$ at the worlds  $w_0$, $w_1$, and  $w_2$   respectively. We have $\overline{t} (w_0)= \{r\}$, $\overline{t} (w_1)= \{p_1,p_2,q\}$, $\overline{t} (w_2)= \{p_2\}$.  In the same way as the last example we can show that $(p_1,w_0)=0$, $(p_2,w_0)=0$,$(q,w_0)=1$, $(p_1,w_1)=1$, $(p_2,w_1)=1$ and  $(p_2,w_2)=1$  are  all the  causes of $r=1$ at $w_0$ by the original and updated definition. However, in the case of modified definition, neither $(p_1,w_1)=1$ nor $(p_2,w_1)=1$ is the causes but $(p_1,w_1)=1 \wedge (p_2,w_1)=1$ is a cause of  $r=1$ at $w_0$. To see this notice that for any choice of $\overline{N}$  we will always have 
  \begin{center}
   $ (\mathcal{K},\overline{t},w_0) \Vdash [(p_1,w_1) \leftarrow {x}', \overline{N} \leftarrow \overline{n}^\ast ]  r=1.$   
  \end{center}
     for any choice of ${x}'$. Thus,  $(p_1,w_1)=1$ is not cause of $r=1$ at $w_0$ by modified definition. Similar argument holds for  $(p_2,w_1)=1$. However, $(p_1,w_1)=1 \wedge (p_2,w_1)=1$  is a cause is showed by setting $\overline{N}=\emptyset$ and $\overline{x}' =(0,0)$. Thus, three definitions of causality need not always match in causal Kripke models.
 \end{example}

The following theorem describes relationship between the three  HP definitions in the causal models. 
 \begin{definition}[\cite{halpern2016actual}, Section 2.2] \label{def:part of cause}
For any event $\alpha$, any variable $X$, and any world $w$, $X=x$ is a  part of cause of $\alpha$ by original (resp.~updated, modified) definition of causality if it is a conjunct in the cause of  $\alpha$ by original (resp.~updated, modified) definition. 
\end{definition}
\begin{theorem}[\cite{halpern2016actual}, Theorem 2.2.3]  \label{thm:reln between three defn}
If $X=x$ is a part of cause of $\alpha$ in $(M,\overline{u})$   according to 
 \begin{enumerate}
     \item the modified HP definition then $X=x$ is a part of cause of $\alpha$ in  $(M,\overline{u})$  according to the original HP definition .
      \item the modified HP definition then $X=x$ is a part of cause of $\alpha$ in $(M,\overline{u})$ according to the updated HP definition.
      \item  the updated HP definition then $X=x$ is a part of cause of $\alpha$ in $(M,\overline{u})$   according to the original HP definition.
 \end{enumerate}
\end{theorem} 
Now, we generalize this result to our framework of causal Kripke  models.

\begin{definition} \label{def:part of cause modal}
For any event $\alpha$, any variable $X$, and any world $w'$, $(X,w)=x$ is a  part of cause of $\alpha$ by original (resp.~updated, modified) definition of causality if it is a conjunct in the cause of  $\alpha$ at $w'$ by original (resp.~updated, modified) definition. 
\end{definition}

\begin{theorem}\label{thm:reln between three defn modal}
If $(X,w)=x$ is a part of cause of $\phi$ in $(\mathcal{K},\overline{t})$ at $w'$ according to 
 \begin{enumerate}
     \item the modified HP definition then $(X,w)=x$  is a part of cause of $\alpha$ in $(\mathcal{K},\overline{t})$ at $w'$ according to the original HP definition.
      \item the modified HP definition then $(X,w)=x$  is a part of cause of $\alpha$ in  $(\mathcal{K},\overline{t})$ at $w'$ according to the updated HP definition.
      \item  the updated HP definition then $(X,w)=x$  is a part of cause of $\alpha$ in  $(\mathcal{K},\overline{t})$ at $w'$ according to the original HP definition.
 \end{enumerate}
 \end{theorem}
 
 \begin{proof}
 For item 1, let $(X,w)=x$ be a part of cause of $\alpha$ in $(\mathcal{K},\overline{t})$  at a world $w'$ according to the modified HP definition, so that there is a cause $\overline{Y}=\overline{y}$ such that $(X,w)=x$ is one of its conjuncts. Then there must exist a value $\overline{x}' \in \mathcal{R}(\overline{Y})$  and a set $\overline{N} \subseteq \mathcal{V} \times W \setminus \overline{Y}$, such that if $(\mathcal{K},\overline{t},w) \vdash  X= n^\ast$ for every $(X,w)=n^\ast \in \overline{N}=\overline{n}^\ast  $,  then  $(\mathcal{K},\overline{t},w') \vdash [\overline{Y} \leftarrow \overline{y}, \overline{N} \leftarrow \overline{n}^\ast]\neg \alpha$. Moreover $\overline{Y}$ is minimal.
 
 We will  show that $(X,w)=x$ is a cause of $\alpha$. If $\overline{Y} =\{(X,w)\}$, then the original HP definition is satisfied by $(\overline{N}, \overline{n}^\ast, x')$ given by the condition $AC2a^m$. If $|\overline{Y}|>1$, then without loss of generality let $\overline{Y} = ((X_1,w_1), (X_2,w_2), \cdots, (X_n,w_n))$ and $(X,w)=(X_1,w_1)$. For any vector $\overline{Y}$, we use $\overline{Y}_{-1}$ to denote all components of $Y$ except the first. 
 
 We will show that $(X_1,w_1)$ is a cause of $\alpha$ in $(\mathcal{K},\overline{t})$ at $w'$ according to the original definition. Since $\overline{Y}=\overline{y}$ is a cause of $\alpha$ in $(\mathcal{K},\overline{t})$ at $w'$ according to the modified definition, by AC1 $(\mathcal{K},\overline{t}, w_1) \vdash X_1=x_1$ and $(\mathcal{K},\overline{t}, w') \vdash \alpha$. Let $\overline{N}' = (\overline{Y}_{-1},\overline{N})$, $\overline{n^\ast}' = (\overline{y'}_{-1},\overline{n^\ast})$, $y'=y_1'$, where $\overline{y}'$ is as given by the modified definition. It is easy to see that $(\mathcal{K},\overline{t},w')\vdash [(X_1,w_1) \leftarrow x_1', \overline{Y}_{-1} \leftarrow \overline{y}_{-1}, \overline{N} \leftarrow \overline{n^\ast}] \neg \alpha$ satisfying condition $AC2a$. Since $(X_1,w_1)$ is single variable, $AC3$ holds trivially. Thus, to complete the proof of $(a)$ we need to show that $AC2b^o$ holds. Suppose $AC2b^o$  does not hold. Then there exists a subset $\overline{Z'} \subseteq \mathcal{V}\times W \setminus (\overline{Y}_{-1} \cup \overline{N})$ of variables  and value $z^\ast$ such that (i) for each $Z \in \overline{Z'}$, $(\mathcal{K},\overline{t},w)\vdash Z=z^\ast$ and (ii) $(\mathcal{K},\overline{t},w')\vdash[(X_1,w_1) \leftarrow x_1, \overline{Y}_{-1} \leftarrow \overline{y}_{-1}, \overline{N} \leftarrow \overline{n^\ast}, \overline{Z'} \leftarrow \overline{z^\ast}] \neg \alpha$. But then $\overline{Y} =\overline{y}$ is not a cause of $\alpha$ according to the modified definition. Indeed, $AC2a^m$ is satisfied  for $\overline{T}'=\overline{Y}_{-1} $
 by setting $\overline{N} = ((X_1,w_1), \overline{N}, \overline{Z'})$ and $\overline{n^\ast} = (x_1, \overline{n^\ast}, \overline{z^\ast})$ and $\overline{t'} = \overline{y'}_{-1}$ violating $AC3$ for $\overline{Y}=\overline{y}$. i.e.~, $\overline{Y}=\overline{y}$ is not a minimal cause by the modified definition as the conjunct obtained by removing $(X_1,w_1)=x_1$ from it is still a cause by the modified definition. 
This is a contradiction. Therefore, $AC2b^{o}$ is valid. 

For item 2, the proof is similar in spirit. In addition to 1, we need to show that if $\overline{Y'} \subseteq \overline{Y}_{-1} $,  $\overline{N'} \subseteq \overline{N}$, and $\overline{Z'} \subseteq \overline{Z}$, then 
\[
(\mathcal{K},\overline{t},w) \vdash [(X,w_1)\leftarrow x_1,  \overline{Y}_{-1} \leftarrow \overline{y}_{-1},  \overline{N'} \leftarrow   \overline{{n^\ast}'},  \overline{Z'} \leftarrow   \overline{{z^\ast}'} \phi ]
\]
If $X'=\emptyset$, then the condition holds since $(X,w)=x$ is a cause of $\phi$ according to the original definition by item 1. In case this condition does not hold for some non-empty $\overline{Y'} \subseteq \overline{Y}_{-1}$, then $\overline{Y} = \overline{y}$ does not satisfy the minimimality condition AC3 of the modified HP definition (in causal Kripke models). 

For item 3, the proof is same as item 1, upto the point where we have to prove $AC2^o$. Suppose there exists $\overline{Z'} \subseteq \overline{Z}$ such that 
\[
(\mathcal{K},\overline{t},w) \vdash [(X,w_1)\leftarrow x_1,  \overline{Y}_{-1} \leftarrow \overline{y}_{-1},  \overline{N'} \leftarrow   \overline{{n^\ast}'},  \overline{Z'} \leftarrow   \overline{{z^\ast}'} \neg \phi ]
\]
then $\overline{Y}_{-1} \leftarrow \overline{y}_{-1}$ satisfies AC2a and $AC2b^u$. Thus, $\overline{Y} = \overline{y}$ does not satisfy the minimimality condition AC3 for the updated definition. Hence proved. 
\end{proof}

 %Proof for item 2 is in similar spirit. Only change is that we need to now show that $AC2b^u$ is valid. Suppose $AC2b^u$ is not valid. Then for some $\overline{(J,v)} \subseteq \overline{(X,w)}_{-1}$, $\overline{N'} \subseteq \overline{N}$, and  $\overline{Z'} \subseteq \overline{Z}$ and values $\overline{x'}$, $\overline{n^\ast}$, and $\overline{z^\ast}$ such that (i)$(\mathcal{K},\overline{t},w')\vdash Z=z^\ast$ for every $(Z,w) \in \overline{Z'}$ and (ii)
 %\[
 %(\mathcal{K},\overline{t},w')\vdash [(X_1,w_1) \leftarrow_1, \overline{(J,v)} \leftarrow \overline{x'}, \overline{N} \leftarrow \overline{n^\ast}, \overline{Z} \leftarrow \overline{z^\ast}  ] \neg \phi.
 %\]
%Again as $AC2a^m$ holds for $\overline{(J,v)}$, the condition AC3 does not hold for $\overline{(X,w)} =\overline{x}$  contradicting that it is cause of $\phi$ (by midfied defitnion). Therefore, $AC2b^u$ is valid. Hence proved.

%Proof of item 3 is similar. 

\section{Soundness and completeness}\label{sec: compl}
In this section we provide the proof of soundness and (weak) completeness of the axiomatization given in Section \ref{ssec:axiomatization}. The proof is a modification of the proof provided in \cite{halpern2016actual} to include the modal operators.
\begin{proof}
Showing that the axiomatization is sound is routine. It is straightforward to verify that all the  axioms except $G$-axiom are valid and that modus ponens and necessitation preserve validity. For the G-axiom, note that the truth of the formula $(X,w)=x$ is independent of the world $w'$ at which it is evaluated. Thus, if it is true at some world, then it is true at all the worlds, in particular true at all the world related to $w$.

To prove that the axiomatization is weakly complete, we show contrapositively that if $\nvdash \psi$ then there exists a model satisfying $\lnot\psi$. As usual, starting with a consistent formula $\varphi$ we obtain a maximal consistent set $\Sigma$ containing all axioms such that $\varphi\in\Sigma$, is closed under $\land$ and consequence, and enjoys the disjunction property (see also the proof of Theorem 5.4.1 in~\cite[Section 5.5]{halpern2016actual}).

Before moving to the details of the proof, we provide a high-level presentation of the argument, to help the reader follow: Given a consistent set of formulas, we can extract the formulas that do not contain modalities. Treating the variables $(X,w)$ and $(X,w)$ (where $w\neq w'$) as simply distinct variables, this set can be seen as a consistent set of formulas for the standard logic of causality presented in \cite{halpern2016actual}, because the axioms in Section \ref{ssec:axiomatization} strictly extend the axioms of the logic in \cite{halpern2016actual}. Then, by the completeness presented in \cite{halpern2016actual}, we get a set of structural equations, which readily provides a set of structural equations over a Kripke model with the empty relation (where $(X,w)$ and $(X,w')$ are now interpreted as the same variable at different points in the Kripke model). By the soundness of the axioms in Section \ref{ssec:axiomatization}, the set of non-modal formulas at each such state is consistent. These consistent sets guarantee that the ``canonical'' model we construct has enough points to interpret all the names that appear in our finite set of formulas. The proof then follows a standard filtration argument to show in the usual way the truth lemma for modal formulas.

Let $\varphi$ be such that $\nvdash\lnot\varphi$. Let us define $W_{\varphi}:=\{w\in W\mid w\text{ appears in }\varphi\}$ and $S_{\varphi}:=\{\psi,\Box\psi \mid \psi\text{ is a subformula of }\varphi\}$. Now let $\Sigma$ be a maximal consistent set of formulas of $\mathbf{L}(W_\varphi)$ that contains $\varphi$. Given axiom C6 in \cite[Section 5.4]{halpern2016actual} and $\Diamond$-axiom and $\Box$-axiom, we can assume without loss of generality that all formulas are generated from $[\overline{Y} \leftarrow \overline{y}]X=x$ and $[\overline{Y} \leftarrow \overline{y}](X,w)=x$ using the connectives $\Diamond,\Box,\land,\lor$ and $\lnot$. Notice that $\Sigma$ ``decides'' the value of variables $(X,w)$ for every $w\in W_{\varphi}$ (that is to say, $(X,w)=x\in\Sigma$ for some $x\in$ for some $x\in\mathcal{R}(X,w)$). Consider the set $B=\{[\overline{Y} \leftarrow \overline{y}](X,w)=x\in\mathbf{L}(W_\varphi)\mid [\overline{Y} \leftarrow \overline{y}](X,w)=x\in\Sigma\}$. By the completeness in \cite[Section 5.5]{halpern2016actual}, it follows that there exists a system of structural equations satisfying the non-modal formulas of $\Sigma$. Using this system we can define a causal Kripke model with domain $W_\varphi$, and empty Kripke relation. By the soundness of this system it follows that for every $w\in W_{\varphi}$ the set $\Sigma_w:=\{[\overline{Y} \leftarrow \overline{y}]X=x\mid [\overline{Y} \leftarrow \overline{y}](X,w)=x\in\Sigma\}\cup B$ (the set $\Sigma_w$ includes the ) is consistent and hence can be extended to an maximal consistent set $\Sigma'_w$.

Let $S=S_{\varphi}\cup\{[\overline{Y} \leftarrow \overline{y}](X,w)=x,[\overline{Y} \leftarrow \overline{y}]X=x\in \mathbf{L}(W_\varphi)\}$. Notice that $S$ is finite. Define an equivalence relation on maximal consistent sets extending $B$ of $\mathbf{L}(W_\varphi)$, $T_1\sim T_2$ if and only if $T_1\cap S=T_2\cap S$. Given that $S$ is finite, there exist finite many equivalence classes. Let $\mathbb{W}$ be the set of equivalence classes, and let $\mathfrak{R}\subseteq\mathbb{W}\times\mathbb{W}$ be defined as $C_1\mathfrak{R}C_2$ if and only if there exists $T_1\in C_1$ and $T_2\in C_2$, such that for all $\psi\in T_2$, $\Diamond \psi\in T_1$. Define a name assignment $i$ such that $i(w)=[\Sigma'_w]$ for $w\in W_\varphi$, and  arbitrarily otherwise. Finally define the structural equations, depending only on variables of $W_\varphi$, exactly as defined in \cite[Section 5.4]{halpern2016actual}. In particular the equations are independent of variables in $W\setminus W_\varphi$, and $f_{(X,w)}(\overline{y})=x$, if and only if $[\overline{Y}\leftarrow\overline{y}]X=x\in T$ for any $T\in i(w)$ (given that $[\overline{Y}\leftarrow\overline{y}]X=x\in S$, this is well defined).

We claim that  $\overline{t},[T]\Vdash \psi$ if and only if $\psi\in T$, for every $\psi\in S$, and maximal consistent set $T$.  

The proof proceeds via induction on the complexity of the formulas. For formulas of the form $[\overline{Y} \leftarrow \overline{y}]X=x$, $[\overline{Y} \leftarrow \overline{y}](X,w)=x$, and for logical connectives the proof is verbatim the same as that of \cite[Section 5.4]{halpern2016actual}. 

Finally, let's show this for the case when $\psi$ is $\Diamond\sigma$. 

First, let's assume that $\overline{t},[T]\Vdash \Diamond\sigma$. Then there exists $C\in\mathbb{W}$ such that $[T]\mathfrak{R}C$ and $\overline{t},C\Vdash \sigma$. By induction hypothesis $\sigma\in T'$, for every $T'\in C$. Since $\sigma\in S$, by the definition of $\mathfrak{R}$, it follows that $\Diamond\sigma\in T$. 

For the converse direction, assume that $\Diamond\sigma\in T$. Notice preliminarily, that since $\Diamond\top\land \Box p\Rightarrow\Diamond p$ is a theorem of classical normal modal logic, then $(\Diamond\top\land \Box([\overline{Y} \leftarrow \overline{y}](X,w)=x))\Rightarrow \Diamond[\overline{Y} \leftarrow \overline{y}](X,w)=x$ is provable in our system. From the G-axiom, this implies that also \begin{equation}\label{eq:comple}
    (\Diamond\top\land ([\overline{Y} \leftarrow \overline{y}](X,w)=x))\Rightarrow \Diamond[\overline{Y} \leftarrow \overline{y}](X,w)=x
\end{equation} is provable. Consider the set $Z_T=\{\tau\in\mathbf{L}(W_\varphi)\mid \Diamond\tau\notin T\}$. Clearly $Z_T$ is an ideal of the free Boolean algebra of the logic. Given that $B\subseteq T$ and $\Diamond\sigma\in T$, it follows that $\Diamond\top\in T$, and by \eqref{eq:comple} it follows that $\Diamond B\subseteq T$ and so $B\cap Z_T=\varnothing$. Hence there exists a maximal consistent set $T'$, extending $B\cap\{\sigma\}$ such that $T'\cap Z_T=\varnothing$. By definition $[T]\mathfrak{R}[T']$, and hence $\overline{t},[T]\Vdash \Diamond\sigma$, as required.

The proof, that the model is recursive, follows again the proof of \cite[Section 5.4]{halpern2016actual}, using the fact that our Kripke frame is finite.
\end{proof}

\iffalse
\subsection{Causal diagrams}

\begin{center}
    \begin{tikzpicture}
       \filldraw[black] (0,0) circle (2 pt);
	\filldraw[black] (-0.5, 2) circle (2 pt);
    \filldraw[black] (2,-0.5) circle (2 pt);
	\filldraw[black] (-0.5,-2) circle (2 pt);
 \filldraw[black] (0.5, 2) circle (2 pt);
    \filldraw[black] (2,0.5) circle (2 pt);
	\filldraw[black] (0.5,-2) circle (2 pt);

 \draw (-0.7, 0) node {$(q,w_0)$};
\draw (-1.2, 2) node {$(p,w_1)$};
\draw (1.2, 2) node {$(r,w_1)$};
\draw (2.7, 0.5) node {$(p,w_2)$};
\draw (2.7, -0.5) node {$(r,w_2)$};
\draw (-0.9, -2.3) node {$(p,w_2)$};
\draw (0.9, -2.3) node {$(r,w_2)$};
\draw[<-](-0.1,0.1)--(-0.5, 1.9);
\draw[<-](0.1,0.1)--(0.5, 1.9);
\draw[<-](-0.1,-0.1)--(-0.5, -1.9);
\draw[<-](0.1,-0.1)--(0.5, -1.9);
\draw[<-](0.1,0)--(1.9, -0.5);
\draw[<-](0.1,0)--(1.9, 0.5);
    \end{tikzpicture}
\end{center}
\fi
\end{document}